\documentclass[10pt,twocolumn,letterpaper]{article}

\usepackage{iccv}
\usepackage{times}
\usepackage{epsfig}
\usepackage{amsmath}
\usepackage{amssymb}
\usepackage{booktabs}
\usepackage{amsthm}
\usepackage{paralist}
\usepackage{algorithm}
\usepackage{algpseudocode}
\usepackage{appendix}

\usepackage[breaklinks=true,bookmarks=false]{hyperref}

\iccvfinalcopy 

\def\iccvPaperID{****} 

\ificcvfinal\fi
\newtheorem{definition}{Definition}

\newtheorem{proposition}{Proposition}
\begin{document}

\title{Encoder Based Lifelong Learning}

\author{Amal Rannen Triki \thanks{Amal Rannen Triki and Rahaf Aljundi contributed equally to this work and listed in alphabetical order.} \thanks{Amal Rannen Triki is also affiliated with Yonsei University. } \quad Rahaf Aljundi$^{ 	\ast}$ \quad  Mathew B. Blaschko \quad  Tinne Tuytelaars\\
KU Leuven\\
KU Leuven, ESAT-PSI, iMinds, Belgium \\
{\tt\small firstname.lastname@esat.kuleuven.be}
}

\maketitle

\begin{abstract}

This paper introduces a new lifelong learning solution where a single model is trained for a sequence of tasks. The main challenge that vision systems face in this context is catastrophic forgetting: as they tend to adapt to the most recently seen task, they lose  performance on the tasks that were learned previously. Our method aims at preserving the knowledge of the previous tasks while learning a new one by using autoencoders. For each task, an under-complete autoencoder is learned, capturing the features that are crucial for its achievement. When a new task is presented to the system, we prevent the reconstructions of the features with these autoencoders from changing, which has the effect of preserving the information on which the previous tasks are mainly relying. At the same time, the features are given space to adjust to the most recent environment as only their projection into a low dimension submanifold is controlled. The proposed system is evaluated on image classification tasks and shows a reduction of forgetting over the state-of-the-art. 
\end{abstract}

\section{Introduction}
Intelligent agents are able to perform remarkably well on individual tasks. However, when exposed to a new task or a new environment, such agents have to be retrained. In this process, they learn the specificity of the new task but tend to loose performance on the tasks they have learned before. 
For instance, imagine an agent that was trained to localize the defects on a set of factory products. Then, when new products are introduced and the agent has to learn to detect the anomalies in these new products, it faces the risk of forgetting the initial recognition task.
This phenomenon is known as {\em catastrophic forgetting}~\cite{mccloskey1989catastrophic,ratcliff1990connectionist,mcclelland1995there,french1999catastrophic,goodfellow2013empirical}. It occurs when the datasets or the tasks are presented  to the model separately and sequentially, as is the  case in a lifelong learning setup~\cite{silver2002task,silver2013lifelong,rusu2016progressive}.

The  main challenge is to make the learned model adapt to new data from a similar or a different environment~\cite{pentina2015lifelong}, without losing knowledge on the previously seen task(s). Most of the classical solutions for this challenge suffer from important drawbacks. Feature extraction (as in~\cite{donahue2014decaf}), where the model / representation learned for the old task is re-used to extract features from the new data without adapting the model parameters, is highly conservative for the old task and suboptimal for the new one. Fine-tuning (as in~\cite{girshick2014rich}),  adapts the model to the new task using the optimal parameters of the old task as initialization. As a result, the model is driven towards the newly seen data but forgets what was learned previously. Joint training (as in~\cite{caruana1998multitask}) is a method where the model is trained jointly on previous and current tasks data. This  solution is optimal for all tasks, but requires the presence of all the data at the same time. Such a requirement can be hard to meet, especially in the era of big data. 


To overcome these drawbacks without the constraint of storing data from the previously seen tasks, two main approaches can be found in the literature. The first, presented in~\cite{li2016learning}, proposes a way to train  convolutional networks, where a shared model is used for the different tasks but with separate classification layers. When a new task is presented, a new classification layer is added. Then, the model is fine-tuned on the data of the new task, with an additional loss that incorporates the knowledge of the old tasks. This loss tries to keep the previous task predictions on the new data unchanged. Such a solution reduces the forgetting but is heavily relying on the new task data.
As  a consequence, it suffers from a build-up of errors when facing a sequence of tasks~\cite{aljundi2016expert}.
The work presented recently in~\cite{kirkpatrick2016overcoming} tackles the problem in a different way. Rather than having a data-oriented analysis, they consider the knowledge gained in the model itself, and transfer it from one task to another in a Bayesian update fashion. 
The method relies on approximating the weight distribution after training the model on the new task. A Gaussian distribution, for which the mean is given by the optimal weights for the first task, and the variance given by the diagonal of the Fisher information matrix is used as an approximation. Such a solution is based on a strong principle and gives interesting results. However, it requires to store a number of parameters that is comparable to the size of the model itself. 

In this work, we propose a compromise between these two methods. Rather than heavily relying on the new task data or requiring a huge amount of parameters to be stored, we introduce the use of autoencoders as a tool to preserve the knowledge from one task while learning another. For each task, an undercomplete autoencoder is trained after training the task model. It captures the most important features for the task objective. When facing a new task, this autoencoder is used to ensure the preservation of those important features. This is achieved by defining a loss on the reconstructions made by the autoencoder, as we will explain in the following sections. In this manner, we only restrict a subset of the features to be unchanged while we give the model the freedom to adapt itself to the new task using the remaining capacity. Figure~\ref{fig:Global_model} displays the model we propose to use. 

\begin{figure}[t]
\centering
\includegraphics[width =\columnwidth]{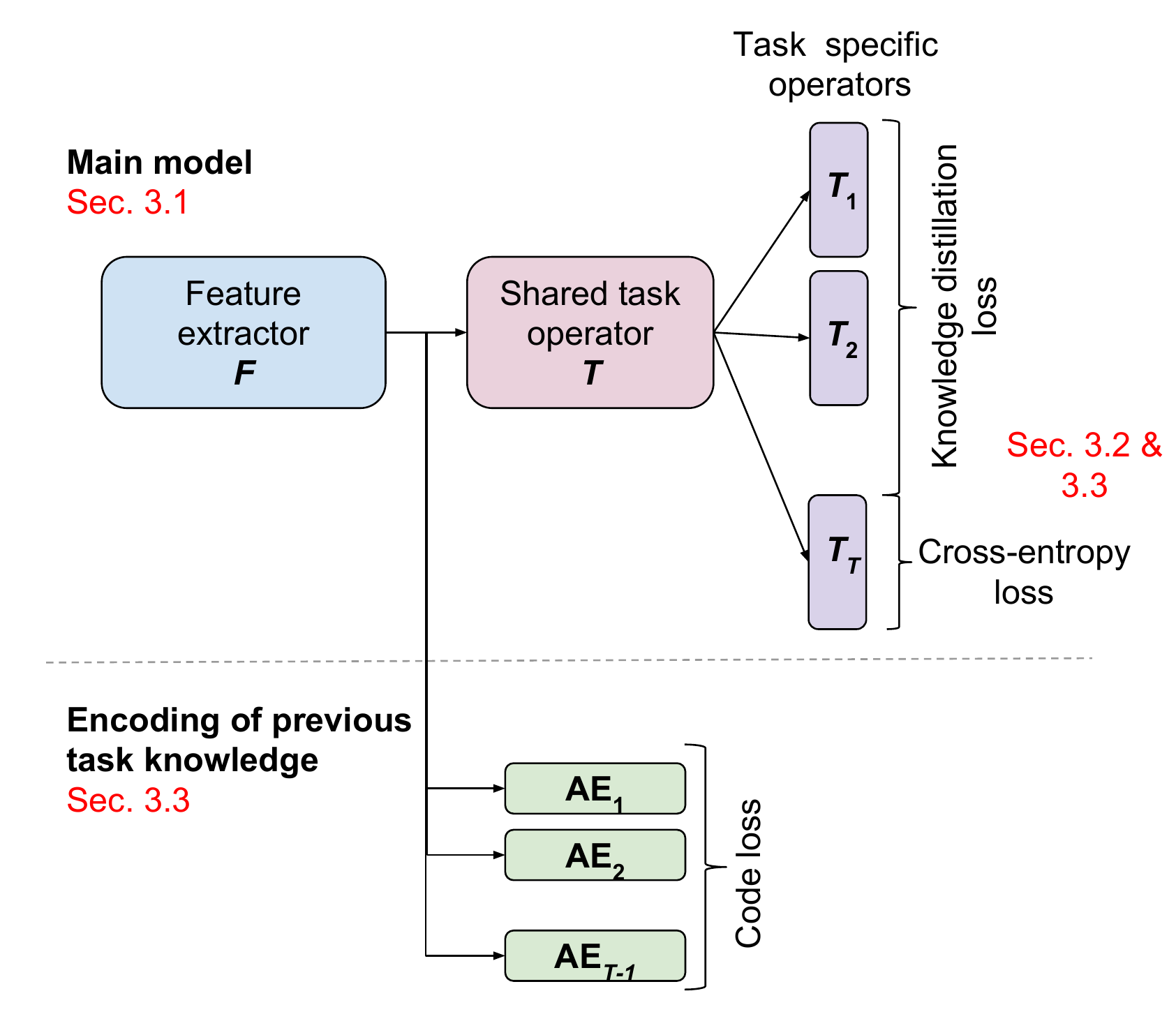}
\caption{Diagram of the proposed model.  Above the dotted line are the model components that are retained during test time, while below the dashed line are components necessary to our improved training scheme.}
\label{fig:Global_model}
\end{figure}

Below, we first give a short description of the most relevant related work in Sec.~\ref{Sec:Related}. Then, in Sec~\ref{Sec:Main}, we describe how to use the autoencoders to avoid catastrophic forgetting and motivate our choice by a short analysis that relates the proposed objective to the joint training  scheme. In Sec.~\ref{Sec:Exp}, we describe the experiments that we conducted, report and discuss their results before concluding in Sec.~\ref{Sec:Conclusion}.

\section{Related work} \label{Sec:Related}

Our end goal is to train a single model that can perform well on multiple tasks, with tasks learned sequentially. This problem is at the intersection of \textbf{Joint training} (or multitask training) and \textbf{Lifelong learning}.
Standard multi-task learning \cite{caruana1998multitask} aims to learn jointly from the data of the multiple tasks and uses inductive bias~\cite{mitchell1980need} in order to integrate the knowledge from the different domains in a single model.  
However, it requires the presence of data from all the tasks during training. In a lifelong learning scenario, on the other hand, tasks are treated in a sequential manner. The aim is then to exploit the knowledge from previous tasks while learning a new one. This knowledge is used to 
\begin{inparaenum}[(i)]
\item preserve the performance on the previously seen data, 
\item improve this knowledge using inductive bias~\cite{mitchell1980need} from the new task data, and
\item act as a regularizer for the new task, which can be beneficial for the performance.
\end{inparaenum}
{\em In this work we aim at preserving the knowledge of the previous tasks and possibly benefiting from this knowledge while learning a new task, without storing data from previous tasks. }


Despite its potential benefits, this problem is under explored. \textbf{Learning without forgetting} (LwF), introduced in~\cite{li2016learning},  
proposes to preserve the previous performance through the knowledge distillation loss introduced in~\cite{hinton2015distilling}. They consider a shared convolutional network between the different tasks in which only the last classification layer is task specific. When encountering a new task, the outputs of the existing classification layers given the new task data are recorded. During training, these outputs are preserved through a modified cross-entropy loss that softens the class probabilities in order to give a higher weight to the small outputs. More details about this loss and the method can be found in Sec.~\ref{Sec:LWF}. This method reduces the forgetting, especially when the datasets come from related manifolds. Nevertheless, it has been shown by \textbf{iCaRL: Incremental Classifier and Representation Learning} ~\cite{rebuffi2016icarl} that LwF suffers from a build up of errors in a sequential scenario where the data comes from the same environment. Similarly, \textbf{Expert-gate}~\cite{aljundi2016expert} shows that the  LwF performance drops when the model is exposed to a sequence of tasks drawn from different distributions. 
\cite{rebuffi2016icarl} proposes to store a selection of the previous tasks data to overcome this issue -- something we try to avoid. 
The goal in~\cite{aljundi2016expert} is to obtain experts for different tasks (instead of a single joint model shared by all tasks). They suggest a model of lifelong learning where experts on individual tasks are added to a network of models sequentially. The challenge then is to decide which expert to launch based on the input. Interestingly, they also use undercomplete autoencoders -- in their case to capture the context of each task based on which a decision is made on which task the test sample belongs to. {\em In this work, we build on top of LwF but reduce the cumulated errors using under-complete autoencoders learned on the optimal representations of the previous tasks.}


Even more recently, another interesting solution to train shared models without access to the previous data, somewhat similar in spirit to our work, has been proposed in~\cite{kirkpatrick2016overcoming}, in the context of reinforcement learning. The main idea of this method, called \textbf{Elastic weight consolidation}, is to constrain the weights $W_i$ while training for a second task with an additional loss $\sum_i \frac{\lambda}{2} F_i(W_i - W^*_{1,i})$ where $W^*_{1,i}$ are the optimal weights of the first task, and $F_i$ the diagonal terms of their Fisher information matrix.
The use of the Fisher matrix prevents the weights that are important for the first task to change much. In our point of view, this method, despite its success, has two drawbacks.  First, the method keeps the weights in a neighborhood of one possible minimizer of the empirical risk of the first task. However, there could be another solution that can give a better compromise between the two tasks.  Second, it needs to store a large number of parameters that grows with the total number of weights and the number of tasks. For these reasons, {\em rather than constraining the weights, we choose to constrain the resulting features, enforcing that those that are important for the previous tasks do not change much}. By constraining only a sub-manifold of the features, we allow the weights to adjust so as to optimize the features for the new task, while preserving those that ensure a good performance on the previous tasks.  

\section{Overcoming forgetting with autoencoders}\label{Sec:Main}

In this work, we consider the problem of training a supervised deep model that can be useful for multiple tasks, in the situation where at each stage the data fed to the network come always from one single task, and the tasks enter in the training scenario successively. 
The best performance for all the tasks simultaneously is achieved when the network is trained on the data from all the considered tasks at the same time (as in joint training). This performance is of course limited by the capacity of the used model, and can be considered an upper bound to what can be achieved in a lifelong learning setting, where the data of previous tasks is no longer accessible when learning a new one.
\subsection{Joint training}
In the following, we will use the notations $\mathcal{X}^{(t)}$ (model input) and $\mathcal{Y}^{(t)}$ (target) for the random variables from which the dataset of the task $t$ is sampled, and $X_i^{(t)}$ and $Y_i^{(t)}$ for the data samples.
When we have access to the data from all $\mathcal{T}$ tasks jointly, the network training aims to control the statistical risk:
\begin{equation}
\sum_{t=1}^\mathcal{T} \mathbb{E}_{(\mathcal{X}^{(t)},\mathcal{Y}^{(t)})}[\ell(f_t(\mathcal{X}^{(t)}),\mathcal{Y}^{(t)})],
\label{eq:StatRisk}
\end{equation}
by minimizing the empirical risk:
\begin{equation}
\sum_{t=1}^\mathcal{T} \frac{1}{N_t} \sum_{i=1}^{N_t} \ell(f_t(X_i^{(t)}),Y_i^{(t)}),
\label{eq:EmpRisk}
\end{equation}
where $N_t$ is the number of samples and $f_t$ the function implemented by the network for task $t$. For most of the commonly used models, we can decompose $f_t$ as $T_t\circ T \circ F$ where:
\begin{itemize}
\item $F$ is a feature extraction function (e.g.\ Convolutional layers in ConvNets)
\item $T_t\circ T$ is a task operator. It can be for example a classifier or a segmentation operator. $T$ is shared among all tasks, while $T_t$ is task specific. (e.g.\ in ConvNets, $T_t$ could be the last fully-connected layer, and $T$ the remaining fully-connected layers.)
\end{itemize}


The upper part of Figure~\ref{fig:Global_model} gives a scheme of this general model. For simplicity, we will focus below on two-task training before generalizing to a multiple task scenario in section~\ref{Sec:Proc}. 


\subsection{Shortcomings of Learning without Forgetting}\label{Sec:LWF}

As a first step, we want to understand the limitations of LwF~\cite{li2016learning}. In that work, it is suggested to replace in Eq~\eqref{eq:StatRisk} $ \ell(T_1\circ\,T\circ\,F(\mathcal{X}^{(1)}),~\mathcal{Y}^{(1)}) $~~~~~with~~~~~~$ \ell(T_1\circ\,T\circ\,F(\mathcal{X}^{(2)}), \\ T^*_1\circ\,T^* \circ\,F^*(\mathcal{X}^{(2)})) $, where $T^*_1\circ\,T^* \circ\,F^*$ is obtained from training the network on the first task.
If we suppose that the model has enough capacity to integrate the knowledge of the first task with a small generalization error, then we can consider that  
\begin{equation}
\mathbb{E}_{(\mathcal{X}^{(1)})}[\ell(T_1\circ\,T\circ\,F(\mathcal{X}^{(1)}),T^*_1\circ\,T^* \circ\,F^*(\mathcal{X}^{(1)}))]
\label{eq:LossApprox}
\end{equation} is a reasonable approximation of $\mathbb{E}_{(\mathcal{X}^{(1)},\mathcal{Y}^{(1)})}[\ell(T_1\circ\,T\circ\,F(\mathcal{X}^{(1)}),\mathcal{Y}^{(1)})]$. However, in order to be able to compute the measure~\eqref{eq:LossApprox} using samples from $\mathcal{X}^{(2)}$, further conditions need to be satisfied. 

In other terms, if we consider that $T_1\circ\,T\circ\,F$ tries to learn an encoding of the data in the target space $\mathcal{X}^{(1)}$, then one can say that the loss of information generated by the use of $\mathcal{X}^{(2)}$ instead of $\mathcal{X}^{(1)}$  is a function of the Kullback-Leibler divergence of the two related probability distributions, or equivalently of their cross-entropy. Thus, if the two data distributions are related, then LwF is likely to lead to high performance. If the condition of the relatedness of the data distributions fails, there is no direct guarantee that the use of $\ell(T_1\circ\,T\circ\,F(\mathcal{X}^{(2)}),T^*_1\circ\,T^* \circ\,F^*(\mathcal{X}^{(2)})) $ will not result in an important loss of information for the first task. Indeed, it has been shown empirically in~\cite{aljundi2016expert} that the use of significantly different data distributions may result in a significant decrease in performance for LwF.

LwF is based on the knowledge distillation loss introduced in~\cite{hinton2015distilling} to reduce the gap resulting from the use of different distributions. In this work, we build on top of the LwF method. In order to make the used approximation less sensitive to the data distributions, we see an opportunity in controlling $\|T_1\circ\,T\circ\,F(\mathcal{X}^{(1)})-T_1\circ\,T\circ\,F(\mathcal{X}^{(2)})\|$. Under mild conditions about the model functions, namely Lipschitz continuity, this control allows us to use $T_1\circ\,T\circ\,F(\mathcal{X}^{(2)})$ instead of $T_1\circ\,T\circ\,F(\mathcal{X}^{(1)})$ to better approximate the first task loss in Eq.~\eqref{eq:StatRisk}.  Note that the condition of continuity on which this observation is based is not restrictive in practice. Indeed, most of the commonly used functions in deep models satisfy this condition (e.g.\ sigmoid, ReLU). 

Our main idea is to learn a submanifold of the representation space $F(\mathcal{X}^{(1)})$ that contains the most informative features for the first task. Once this submanifold is identified, if the projection of the features  $F(\mathcal{X}^{(2)})$ onto this submanifold do not change much during the training of a second task,
then two consequences follow: 
\begin{inparaenum}[(i)]
\item $F(\mathcal{X}^{(2)})$ will stay informative for the first task 
during the training, and  
\item at the same time there is room to adjust to the second task 
as only its projection in the learned submanifold is controlled.
\end{inparaenum}
Figure~\ref{fig:preserve_feat} gives a simplified visualization of this mechanism. In the next paragraphs, we propose a method to learn the submanifold of informative features for a given task using autoencoders.
\begin{figure}[t]
\centering
\includegraphics[width = 0.9\columnwidth]{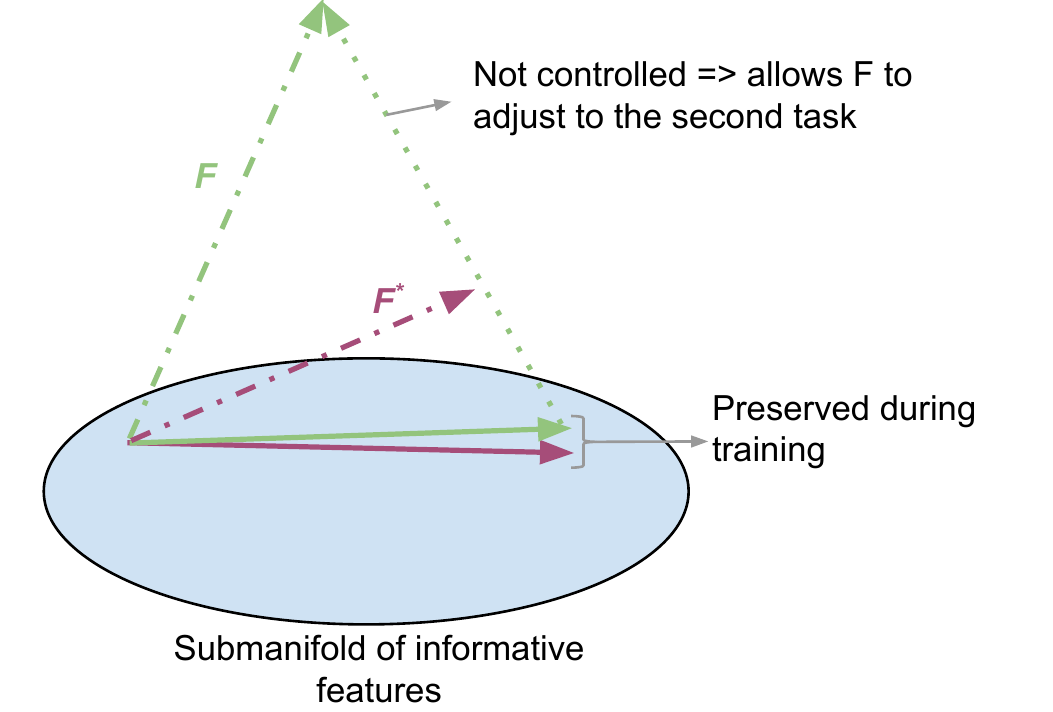}
\caption{Preservation of the features that are important for task 1 while training on task 2. During training, we enforce the projection of $F$ into the submanifold that captures these important features to stay close to the projection of $F^*$, the optimal features for the first task. The part of $F$ that is not meaningful for the first task is allowed to adjust to the variations of the second task.}
\label{fig:preserve_feat}
 \vspace*{-0.2cm} 
\end{figure}





\subsection{Informative feature preservation}

When beginning to train the second task, the feature extractor $F^*$ of the model is optimized for the first task. A feature extraction type of approach would keep this operator unchanged in order to preserve the performance on the previous task. This is, however, overly conservative, and usually suboptimal for the new task. Rather than preserving \emph{all} the features during training, our main idea is to preserve only the features that are the most informative for the first task while giving more flexibility for the other features in order to improve the performance on the second task. An autoencoder~\cite{bourlard1988auto} trained on the representation of the first task data obtained from an optimized model can be used to capture the most important features for this task. 
\vspace{-6pt}
\subsubsection{Learning the informative submanifold with Autoencoders} \label{Sec:AE}
An autoencoder is a neural network that is trained to reconstruct its input \cite{Goodfellow-et-al-2016-Book}. The network operates a projection $r$ that can be decomposed in an encoding and a decoding function. The optimal weights are usually obtained by minimizing the mean $\ell_2$ distance between the inputs and their reconstructions. If the dimension of the code is smaller than the dimension of the input (i.e.\ if the autoencoder is  undercomplete), the autoencoder captures the submanifold that represents the best the structure of the input data.  More precisely, we choose to use a two-layer network with a sigmoid activation in the hidden layer: $r(x) = W_{dec}\sigma(W_{enc}x). $ Figure~\ref{fig:AE} shows a general scheme of such an  autoencoder.

\begin{figure}[t]
\centering
\includegraphics[width = \columnwidth]{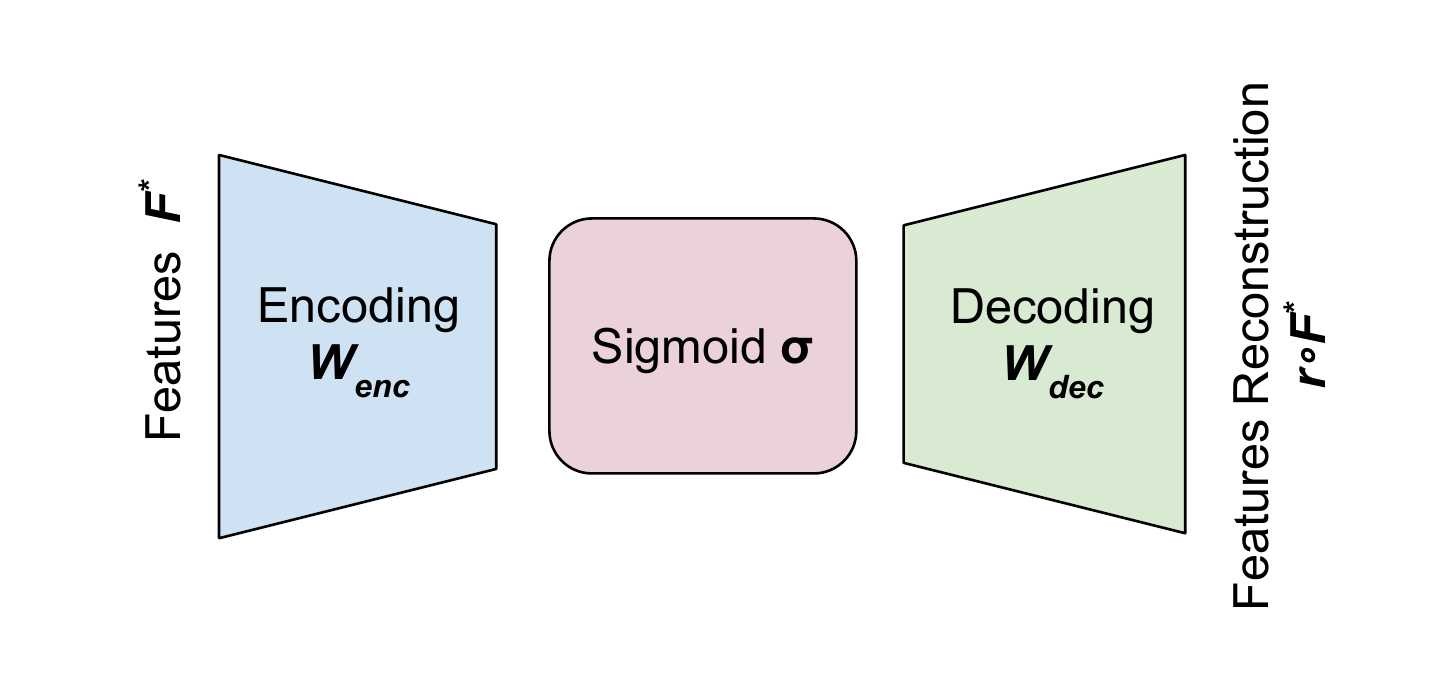}
\caption{Scheme of an undercomplete autoencoder trained to capture the important features submanifold.}
\label{fig:AE}
 \vspace*{-0.2cm} 
\end{figure}

Here, our aim is to obtain through the autoencoder a sub-manifold that captures the information that is not only important to reconstruct the features (output of the feature extraction operator $F^*$) of the first task, but also important for the task operator ($T^*_1\circ T^*$). The objective is:
\begin{align}
\arg\min_{r} \mathbb{E}_{(\mathcal{X}^{(1)},\mathcal{Y}^{(1)})}[ &\lambda\|r(F^*(\mathcal{X}^{(1)}))- F^*(\mathcal{X}^{(1)})\|_2  \label{eq:LossAE}\\ 
&+ \ell(T_1^*\circ T^*(r(F^*(\mathcal{X}^{(1)}))),\mathcal{Y}^{(1)}) ],
\nonumber
\end{align}
where $\ell$ is the loss function used to train the model on the first task data.  $\lambda$ is a hyper-parameter that controls the compromise between the two terms in this loss. 
In this manner, the autoencoder represents the variations that are needed to reconstruct the input and at the same time contain the information that is required by the task operator.
\vspace{-6pt}
\subsubsection{Representation control with separate task operators} \label{Sec:intuition}

To explain how we use these autoencoders,
we start with the simple case where there is no task operator shared among all tasks (i.e., $T = \emptyset$). The model is then composed of a common feature extractor $F$, and a task specific operator for each task $T_t$. 
Each time a new task is presented to the model, the corresponding task operator is then optimized. However, in order to achieve the aim of lifelong learning, we want to also update the feature extractor without damaging the performance of the model on the previously seen tasks. 


In a two task scenario, after training the first task, we have $T_1^*$ and $F^*$  optimized for that task. Then, we train an undercomplete autoencoder using $F^*(X_i^{(1)})$ minimizing the empirical risk corresponding to~\eqref{eq:LossAE}. The optimal performance for the first task, knowing that the operator $T_1$ is kept equal to $T_1^*$, is obtained with $F$ equal to $F^*$.

Nevertheless, preventing $F$ from changing will lead to suboptimal performance on the second task. The idea here is to keep only the projection of $F$ into the manifold represented by the autoencoder ($r\circ~F$) unchanged. 
The second term of Eq.~\eqref{eq:LossAE} explicitly enforces $r$ to represent the submanifold needed for good performance on task 1. Thus, controlling the distance between $r\circ F$ and $r\circ F^*$ will preserve the necessary information for task 1. From the undercompleteness of the encoder, $r$ projects the features into a lower dimensional manifold, and by controlling only the distance between the reconstructions, we give the features flexibility to adapt to the second task variations (cf.\ Figure~\ref{fig:preserve_feat}).

\subsubsection{Representation control with shared task operator}\label{Sec:Shared}

We now consider the model presented in Figure~\ref{fig:Global_model} where a part of the task operator is shared among the tasks as in the setting used in LwF~\cite{li2016learning}. This is clearly a preferrable architecture in a lifelong learning setting, as the memory increase when adding a new task is much lower. Our main idea is to start from the loss used in the LwF method and add an additional term coming from the idea presented in Sec.~\ref{Sec:intuition}. Thus, in a two task scenario, in addition to the loss used for the second task, we propose to use two constraints:
\begin{enumerate}
\item The first constraint is the knowledge distillation loss ($\ell_{dist}$) used in~\cite{li2016learning}. If 
$
\hat{\mathcal{Y}} := T_1\circ T\circ F(\mathcal{X}^{(2)}) 
$
and 
$
\mathcal{Y}^* = T^*_1\circ T^*\circ F^*(\mathcal{X}^{(2)}) 
$
then:
\begin{equation}
\ell_{dist}(\hat{\mathcal{Y}} ,\mathcal{Y}^*) = -\langle \mathcal{Z}^*,\log \hat{\mathcal{Z}}\rangle
\end{equation}
where $\log$ is operated entry-wise and
\begin{equation}
\mathcal{Z}^*_i = \frac{\mathcal{Y}_i^{*1/\theta}}{\sum_j\mathcal{Y}_j^{*1/\theta}} \text{ and } \hat{\mathcal{Z}}_i = \frac{\hat{\mathcal{Y}}_i^{1/\theta}}{\sum_j\hat{\mathcal{Y}}_j^{1/\theta}} 
\end{equation}
The application of a high temperature $\theta$ increases the small values of the output and reduces the weight of the high values. This mitigates the influence of the use of different data distributions.
\item The second constraint is related to the preservation of the reconstructions of the second task features ($r\circ F^*(\mathcal{X}^{(2)})  $). The goal of this constraint is to keep $r\circ F$ close to $r\circ F^*$ as explained in Sec.~\ref{Sec:intuition}.
\end{enumerate}
For the second constraint, rather than controlling the distance between the reconstructions, we will here constrain the codes $\sigma(W_{enc}\cdot )$. From sub-multiplicity of the Frobenius norm, we have:
\begin{equation}
\|r(x_1) - r(x_2)\|_2 \le \|W_{dec}\|_F\|\sigma(W_{enc}x_1) - \sigma(W_{enc}x_2)\|_2. \nonumber
\end{equation}
The advantage of using the codes is their lower dimension. As the codes or reconstructions need to be recorded before beginning the training on the second task, using the codes will result in a better usage of the memory.

Finally, the objective for the training of the second task is the following: 
\begin{align}
\mathcal{R} &= \mathbb{E}[\ell(T_2\circ T\circ  F(\mathcal{X}^{(2)}), \mathcal{Y}^{(2)}) )\nonumber \\
&+  \ell_{dist}(T_1\circ T\circ F(\mathcal{X}^{(2)}), T_1^*\circ T^*\circ F^*(\mathcal{X}^{(2)})  ) \nonumber  \\
& + \frac{\alpha}{2} \| \sigma(W_{enc}F(\mathcal{X}^{(2)})) - \sigma(W_{enc}F^*(\mathcal{X}^{(2)})) \|_2^2] \label{eq:CodeLoss}.
\end{align}
The choice of the parameter $\alpha$ will be done through model selection. An analysis of this objective is given in the appendix, giving a bound on the difference between $\mathbb{E}[\ell(T_2\circ~T\circ~F(\mathcal{X}^{(2)}),\mathcal{Y}^{(2)}))
+  \ell(T_1\circ T\circ F(\mathcal{X}^{(2)}), T_1^*\circ T^*\circ F^*(\mathcal{X}^{(2)})  )]$ and the statistical risk in a joint-training setting~\eqref{eq:StatRisk}. It shows that~\eqref{eq:CodeLoss} effectively controls this bound.

\subsection{Training procedure}\label{Sec:Proc}
The proposed method in Sec.~\ref{Sec:Shared} generalizes easily to a sequence of tasks. An autoencoder is then trained after each task. Even if the needed memory will grow linearly with the number of tasks, the memory required by an autoencoder is a small fraction of that required by the global model. For example, in the case of AlexNet as a base model, an autoencoder comprises only around 1.5\% of the memory. 

In practice, the empirical risk is minimized: 
\begin{align}
R_N&= \frac{1}{N} \sum_{i=1}^N \Biggl( \ell(T_T\circ T\circ F(X_i^{(\mathcal{T})}), Y_i^{(\mathcal{T})}) \nonumber \\
&+ \sum_{t=1}^{\mathcal{T}-1} \ell_{dist}(T_t\circ T\circ F(X_i^{(\mathcal{T})}), T_t^*\circ T^*\circ  F^*(X_i^{(\mathcal{T})})  ) \nonumber  \\
& +\sum_{t=1}^{\mathcal{T}-1}\frac{\alpha_t}{2} \| \sigma(W_{enc,t}F(X_i^{(\mathcal{T})})) - \sigma(W_{enc,t}F^*(X_i^{(\mathcal{T})})) \|_2^2\Biggr). \label{eq:MultiEmp}
\end{align}
The training is done using stochastic gradient descent (SGD)~\cite{bottou2010large}.
The autoencoder training is also done by SGD but with an adaptive gradient method, AdaDelta~\cite{zeiler2012adadelta} which alleviates the need for setting the learning rates and has nice optimization properties.
Algorithm \ref{algo} shows the main steps of the proposed method.
 \begin{algorithm}[t]
\caption{Encoder based Lifelong Learning}\label{algo}
  \begin{algorithmic}[1]
  \Statex \textbf{{Input}} :
  \Statex $F^*$ shared feature extractor; $T^*$: shared task operator; $\left\{T_t\right\}_{t=1..\mathcal{T}-1}$ previous task operators; 
  \Statex $\left\{W_{enc,t}\right\}_{t=1..\mathcal{T}-1}$ previous task encoders;
  \Statex $(X^{(\mathcal{T})}, Y^{(\mathcal{T})})$ training data and ground truth of the new task $\mathcal{T}$;
  \Statex $\alpha_t$ and $\lambda$ //hyper parameters
  \Statex \textbf{{Initialization}} :
  
	\State $Y_t^* = T_t^*\circ T^*\circ F^*(X^{(\mathcal{T})})$	   //record task targets  
    \State $C_t^* = \sigma(W_{enc,t}F^*(X^{(\mathcal{T})}))$   //record new data codes
  \State $T_\mathcal{T} \leftarrow$ Init($|Y^{(\mathcal{T})}|$) // initialize new task operator
  
  \Statex \textbf{{Training}} :
  \State  $\hat{Y_t} = T_t\circ T\circ F(X^{(\mathcal{T})})$ // task outputs
  \State   $\hat{C_t} = \sigma(W_{enc,t}F(X^{(\mathcal{T})}))$ //current codes
\State $ T_t^*, T^*, F^* \leftarrow \arg\min_{T_t, T, F}[\ell(\hat{Y_{\mathcal{T}}} , Y^{(\mathcal{T})}) ) $
\Statex $+  \sum_{t=1}^{\mathcal{T}-1} \ell_{dist}(\hat{Y_t}, Y_t^* ) +\sum_{t=1}^{\mathcal{T}-1}\frac{\alpha_t}{2} \| \hat{C_t} -  C_t^*\|^2]$
   \State $(W_{enc,\mathcal{T}},W_{dec,\mathcal{T}}) \leftarrow$ autoencoder( $T_\mathcal{T}^*, T^*, F^*,$
  \Statex $X^{(\mathcal{T})}, Y^{(\mathcal{T})}; \lambda$ ) // minimizes Eq.~\eqref{eq:LossAE}

  \end{algorithmic}
 
  \end{algorithm}

\section{Experiments}\label{Sec:Exp}
\begin{table*}[ht]
\centering
\tabcolsep=0.11cm
\footnotesize
\begin{tabular}{|l|ll|ll||ll|ll||ll|ll||ll|}
\hline
& \multicolumn{4}{|c|}{ImageNet $\rightarrow$ Scenes} &   \multicolumn{4}{c|}{ImageNet $\rightarrow$ Birds} & \multicolumn{4}{c|}{ImageNet $\rightarrow$ Flowers} & \multicolumn{2}{c|}{Average loss} \\ \hline
& \multicolumn{2}{c|}{Acc. on Task1} & \multicolumn{2}{c|}{Acc. on Task2} & \multicolumn{2}{c|}{Acc. on Task1} & \multicolumn{2}{c|}{Acc. on Task2} & \multicolumn{2}{c|}{Acc. on Task1} & \multicolumn{2}{c|}{Acc. on Task2} & Task 1 & Task 2 \\ \hline
Finetuning& 48.0\% &(-9\%)         & 65.0\% &(ref)	&  41.3\% & (-15.7\%) 	&  59.0\% & (ref)	& 50.8\% & (-6.2\%) & 86.4\% & (ref) & -10.3\% & (ref) \\ 
Feature extraction& 57.0\% & (ref) & 60.6\% &(-4.4\%)	&   57.0\% & (ref) 		& 51.6\% & (-7.4\%) & 57.0\% & (ref)	& 84.6\% & (-1.8\%) & (ref) & -4.5\%\\ 
LwF  & 55.4\% &(-1.6\%)              & 65.0\% &(-0\%)	&  54.4\% & (-2.6\%) 	&  58.9\% & (-0.1\%)& 55.6\% & (-1.4\%) & 85.9\% & (-0.5\%) & -1.9\% & -0.2\%\\ 
{\bf Ours} & 56.3\% & (-0.7\%)       & 64.9\% &(-0.1\%) &  55.3\% & (-1.7\%) 	&  58.2\% & (-0.8\%)& 56.5\% & (-0.5\%) & 86.2\% & (-0.2\%) & -1.0\% & -0.4 \% \\ 
{\bf Ours} - separate FCs& 57.0\% & (-0\%) & 65.9\% &(+0.9\%)&  57.0\% & (-0\%)	&  57.7\% & (-1.3\%)& 56.5\% & (-0.5\%) & 86.4\% & (-0\%)   & -0.2\% & -0.1 \% \\ \hline
\end{tabular}
\caption{Classification accuracy for the Two Tasks scenario starting from ImageNet. For the first task, the reference performance is given by Feature extraction. For the second task, we consider Finetuning as the reference as it is the best that can be achieved by one task alone.}
\label{tab:res_twotasks_imagenet}
\end{table*}

We  compare our method against the state-of-the-art and several baselines on image classification tasks.
%
We consider sets of 2 and 3 tasks learned sequentially, in two settings: 1) when the first task is a large dataset, and 2) when the first task is a small dataset.

\paragraph{Architecture}
We experiment with AlexNet \cite{krizhevsky2012imagenet} as our network architecture due to its widespread use and similarity to other popular architectures. The feature extraction block $F$ corresponds to the convolutional layers. By default, the shared task operator $T$ corresponds to all but the last fully connected layers (i.e., $fc6$ and $fc7$), while the task-specific part $T_i$ contains the last classification layer ($fc8$). During the training, we used an $\alpha$ of $10^{-3}$ for ImageNet and $10^{-2}$ for the rest of the tasks. Note that this parameter sets the trade off between the allowed forgetting on the previous task and the performance on the new task.

For the autoencoders, we use a very shallow architecture, to keep their memory footprint low. Both the encoding as well as the decoding consist of a single fully connected layer, with a sigmoid as non-linearity in between. The dimensionality of the codes is 100 for all datasets, except for ImageNet where we use a code size of 300. The size of the autoencoder is 3MB, compared to 250MB for the size of the network model. The training of the autoencoders is done using AdaDelta as explained in Sec.~\ref{Sec:Proc}. During training of the autoencoders, we use a hyperparameter $\lambda$ (cf. Eq.~\eqref{eq:LossAE}) to find a compromise between the reconstruction error and the classification error. $\lambda$ is tuned manually in order to allow the convergence of the code loss and the classification loss on the training data. Figure~\ref{fig:AE_loss} shows the evolution of these losses for the training and validation samples of ImageNet during the training of an autoencoder based on the $conv5$ features extracted with AlexNet. In our experiments, $\lambda$ is set to $10^{-6}$ in all cases.

\begin{figure}[t]
\centering
 \vspace*{-0.7cm} 
\includegraphics[width=0.4\textwidth,height=4cm]{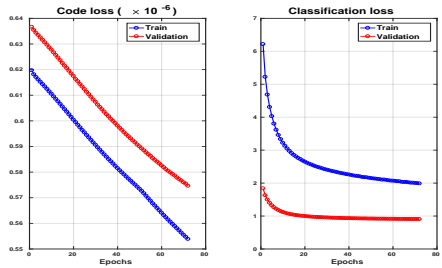}
  \caption{Training of AlexNet based autoencoder for ImageNet - The objective makes the code loss \emph{and} the classification loss decrease. The training is stopped when we observe a convergence of the classification loss.  }
   
  \label{fig:AE_loss}
  \vspace*{-0.5cm} 
\end{figure}

\paragraph{Datasets}\label{Sec:dataset}
We use multiple datasets of moderate size: MIT \textit{Scenes} \cite{quattoni2009recognizing} for indoor scene classification (5,360 samples), Caltech-UCSD \textit{Birds}~\cite{WelinderEtal2010} for fine-grained bird classification (5,994 samples) and Oxford \textit{Flowers}~\cite{Nilsback08} for fine-grained flower classification (2,040 samples). We excluded Pascal VOC as it is very similar to ImageNet, with subcategories of some of the VOC labels corresponding to ImageNet classes. These datasets were also used in both~\cite{aljundi2016expert} and~\cite{li2016learning}.

For the scenario based on a large initial dataset, we start from ImageNet (LSVRC 2012 subset)~\cite{russakovsky2015imagenet}, which has more than 1 million training images.  For the small dataset scenario, we start from Oxford \textit{Flowers}, which has only 2,040 training and validation samples. 

The reported results are obtained with respect to the test sets of Scenes, Birds and Flowers, and on the validation set of ImageNet. 
As in LwF~\cite{li2016learning}, we need to record the targets for the old tasks before starting the training procedure for a new task. Here, we perform an offline augmentation with 10 variants of each sample (different crops and flips). This setting differs slightly from what has been done in~\cite{li2016learning}, which explains the higher performance on the individual tasks in our experiments.  We therefore compare against a stronger baseline than the accuracies reported in~\cite{li2016learning}. 
\paragraph{Compared Methods}
We compare our method 
(\textbf{Ours}) with Learning without Forgetting (\textbf{LwF})~\cite{li2016learning}, which represents the current state-of-the-art. 
Additionally, we consider two baselines: \textbf{Finetuning}, where each model (incl. $F$ and $T$) is learned for the new task using the previous task model as initialization, and \textbf{Feature extraction}, where the weights of the previous task model ($F$ and $T$) are fixed and only the classification layer ($T_t$) is learned for each new task. Further, we also report results for a variant of our method, 
\textbf{Ours - separate FCs} where we only share the representation layers ($F$) while each task has its own fully connected layers (i.e., $T = \emptyset$ and $T_i = \{fc6-fc7-fc8\}$). This variant aims at finding a universal representation for the current sequence of tasks while allowing each task to have its own fully connected layers. With less sharing, the risk of forgetting is reduced, at the cost of a higher memory consumption and less regularization for new tasks. Note that in this case the task autoencoders can be used at test time to activate only the fully connected layers of the task that a test sample belongs to, in a similar manner to what was done in \cite{aljundi2016expert}.

\paragraph{Setups}\label{sec:setup}
We consider sequences of 2 and 3 tasks. In the {\bf Two Tasks} setup, we are given a model trained on one previously seen task and then add a second task to learn. This follows the experimental setup of LwF~\cite{li2016learning}. In their work, all the tested scenarios start from a large dataset, ImageNet. Here we also study the effect of starting from a small dataset, Flowers.\footnote{Due to the small size of the Flowers dataset, we use a network pretrained on ImageNet as initialization for training the first task model. The main difference hence lies in the fact that in this case we do not care about forgetting ImageNet.}
Further, we also consider a setup, involving {\bf Three  Tasks}. First, we use a sequence of tasks starting from ImageNet, i.e.~ImageNet $\rightarrow$ Scenes $\rightarrow$ Birds. Additionally, we consider Flowers as a first task in the sequence Flowers $\rightarrow$ Scenes $\rightarrow$ Birds. Note that this is different from what was conducted  in \cite{li2016learning} where the sequences  were only composed of  splits of  one dataset  i.e.~one task overall.


\begin{table*}[ht]
\centering
\tabcolsep=0.11cm
\small
\begin{tabular}{|l|ll|ll||ll|ll||ll|}
\hline
& \multicolumn{4}{|c|}{Flowers $\rightarrow$ Scenes} &   \multicolumn{4}{c|}{Flowers $\rightarrow$ Birds}  & \multicolumn{2}{c|}{Average loss}   \\ \hline
& \multicolumn{2}{c|}{Acc. on Task1} & \multicolumn{2}{c|}{Acc. on Task2} & \multicolumn{2}{c|}{Acc. on Task1} & \multicolumn{2}{c|}{Acc. on Task2} &  Task 1 &  Task 2 \\ \hline
Finetuning  		& 61.6\% & (-24.8\%)	& 63.9\% & (ref) 	& 66.6\%      & (-19.8\%)		 & 57.5\% & (ref) &-22.3\% &(ref)\\ 
Feature extraction  & 86.4\% & (ref) 		& 59.6\% & (-4.3\%) &  86.4\% & (ref) & 48.6\% & (-8.9\%) &(ref) &-6.6\%\\ 
LwF  				& 83.7\% & (-2.7\%) 	& 62.2\% & (-1.7\%)	&  82.0 \%& (-4.4\%) &  52.2 \%& (-5.3\%)  & -3.6\%&-3.5\%\\ 
{\bf Ours} 			& 84.9\% & (-1.5\%) 	& 62.3\% & (-1.6\%) &  83.0\% & (-3.4\%) & 52.0 \% & (-5.5\%) & -2.4\%&-3.5\%\\ 
{\bf Ours} - separate FCs & 86.4\% & (-0\%) 	& 63.0\% & (-0.9\%) &  85.4\% & (-1.0\%) & 55.1\% & (-2.4\%) &-0.5\%&-1.6\%\\ \hline
\end{tabular}
\caption{Classification accuracy for the Two Tasks scenario starting from Flowers. For the first task, the reference performance is given by Feature extraction. For the second task, we consider Finetuning as reference as it is the best that can be achieved by one task alone.}
\label{tab:res_twotasks_flower}
\vspace*{-0.1cm} 
\end{table*}

\begin{table}[h]
\centering
\tabcolsep=0.11cm
\small
\begin{tabular}{|l|lll|c|}
\hline
& ImageNet & Scenes & Birds & Average Acc.   \\

  \hline
Finetuning& 37.5\% &45.6\% & \textbf{ 58.1}\% &47.2\%\\  
LwF & 53.3\% & 63.5\% & 57.2 \%& 58.0\% \\
{\bf Ours} & \textbf{54.9}\% &\textbf{64.7}\% & 56.9\% &\textbf{58.8}\% \\ 
\hline 
\end{tabular}
\caption{Classification accuracy for the Three Tasks scenario starting from ImageNet. \textbf{Ours} achieves the best trade off between the tasks in the sequence with less forgetting to the previous tasks.}
\label{tab:res_seq_imgnet}
\vspace*{-0.1cm} 
\end{table}

\begin{table}[h]
\centering
\tabcolsep=0.11cm
\small
\begin{tabular}{|l|lll|c|}
\hline
& Flowers & Scenes & Birds&Average Acc.  \\
\hline
Finetuning &51.2\% &48.1\% &{\bf 58.5\%} &51.6\%\\ 
 
LwF    & 81.1\% &59.1\% & 52.3\% &64.1\% \\
{\bf Ours} & {\bf 82.8\%} &{\bf 61.2 \%} &  51.2\% & {\bf 65.0\%} \\ \hline 

\end{tabular}
\caption{Classification accuracy for the Three Tasks scenario starting from Flowers. \textbf{Ours} achieves the best trade off between the tasks in the sequence with less forgetting to the previous tasks.}
\label{tab:res_seq_flower}

\end{table}




\paragraph{Results}
Table \ref{tab:res_twotasks_imagenet} shows, for the different compared methods, the achieved performance on the Two Tasks scenario with ImageNet as the first task.  While \textbf{Finetuning} is optimal for the second task, it shows the most forgetting of the first task. The performance on the second task is on average comparable for all methods except for \textbf{Feature extraction}. Since the \textbf{Feature extraction} baseline doesn't allow the weights of the model to change and only optimizes the last fully connected layers, its performance on the second task is suboptimal and significantly lower than the other methods. Naturally the performance of the previous task is kept unchanged in this case.
\textbf{Ours - separate FCs} shows  the best compromise between the two tasks. The performance of the first task is highly preserved and at the same time the performance of the second task is comparable or better to the methods with shared FCs. This variant of our  method has a higher capacity as it allocates separate fully connected layers for each task, yet its memory consumption increases more rapidly as tasks are added. Our method with a complete shared model \textbf{Ours} systematically outperforms the \textbf{LwF} method on the previous task and on average achieves a similar performance on the second task. 

When we start from a smaller dataset, Flowers, the same trends can be observed, but with larger differences in accuracy 
(Table~\ref{tab:res_twotasks_flower}).
The performance on the second task is lower than that achieved with ImageNet as a starting point for all the compared methods. This is explained by the fact that the representation obtained from ImageNet is more meaningful for the different tasks than what has been finetuned for Flowers. Differently from the ImageNet starting case, \textbf{Ours - separate FCs}  achieves a considerably better performance on the second task than \textbf{Ours} and  \textbf{LwF} while preserving the previous task performance.  \textbf{Finetuning} shows the best performance on the second task while suffering from severe forgetting on the previous task. Indeed the pair of tasks here is of a different distribution and finding a compromise between the two is a challenging problem.
As in the previous case \textbf{Ours} reduces the forgetting of  \textbf{LwF} while achieving a similar average performance on the second task.

Overall, the \textbf{Ours-separate FCs} achieves the best performance on the different pairs of tasks. However, it requires allocating seprate fully connected layers for each task which requires a lot of memory. Thus, for the sequential experiments we focus on the shared model scenario.
   
In Table \ref{tab:res_seq_imgnet} we report the performance achieved by \textbf{Ours}, \textbf{LwF} and \textbf{Finetuning} for the sequence of  ImageNet $\rightarrow$ Scenes $\rightarrow$ Birds.  As expected the \textbf{Finetuning} baseline suffers from severe forgetting on the previous tasks. The performance on  ImageNet (the first task) drops from 57\% to 37.5\% after finetuning on the third task. As this baseline does not consider the previous tasks in its training procedure it has the advantage of achieving the best performance on the last task in the sequence, as expected.  

\textbf{Ours}  continually  reduces forgetting compared to  \textbf{LwF} on the previous tasks while showing a comparable performance on the new task in the sequence. For example, \textbf{Ours} achieves 54.9\% on ImageNet compared to 53.3\% by LwF. Similar conclusions can be drawn regarding the sequential scenario starting from Flowers as reported in Table~\ref{tab:res_seq_flower}.

\begin{figure}[t]
\centering
\includegraphics[width=0.28\textwidth]{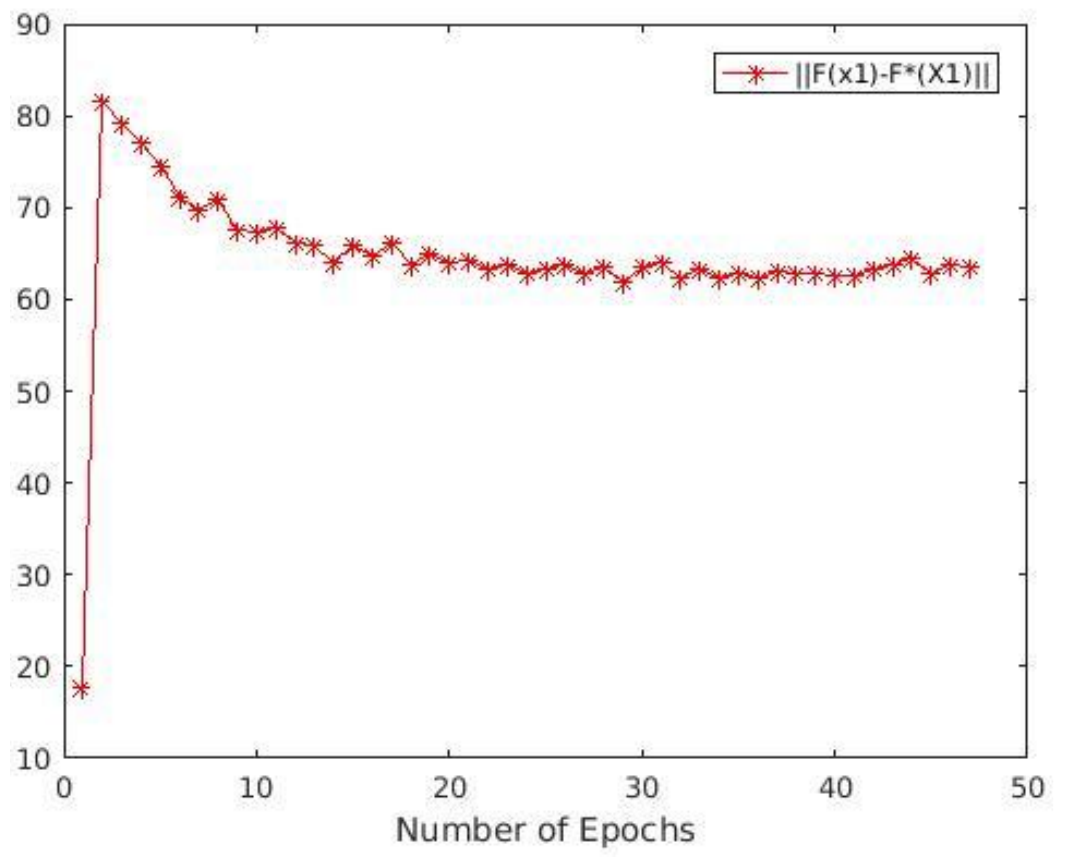}
  \caption{Distance between the representation obtained during training of the Birds task given  Flowers samples and the original representation for Flowers network. Starting from the Scenes network trained using our method after Flowers.}
   
  \label{fig:dis}
  \vspace*{-0.5cm} 
\end{figure}

\paragraph{Behavior Analysis}\label{sec:analysis}
To examine the effect of our representation control over the learning process, we perform an analysis  on the distance between the representation obtained over the learning procedure and the one that is optimal for the first task. We use Flower $\rightarrow$ Scenes$\rightarrow$ Birds as a test case and compute the distance for each epoch between the current features of the Flowers dataset $F(\mathcal{X}^1)$ and that obtained by the initial Flowers network $F^*(\mathcal{X}^1)$, as shown in Figure~\ref{fig:dis}.
In the beginning  of the training, the leading term in Eq.~\eqref{eq:MultiEmp} is the loss related to the new task $\mathcal{T}$. Thus, in the first stage, the model is driven towards optimizing the performance of the most recent task. This results in a quick loss of performance for the previous tasks, and an increase in the other loss terms of the objective. Then, the second stage kicks in. 
In this stage, all the terms of Eq.~\eqref{eq:MultiEmp} contribute and the model is pushed towards recovering its performance for the previous tasks while continuing improving for the most recent one. Gradually, $F(\mathcal{X}^1)$ gets again closer to $F^*(\mathcal{X}^1)$, until a new equilibrium is reached.

\section{Conclusions and future work} \label{Sec:Conclusion}
Strategies for efficient lifelong learning is still an open research problem. In this work we tackled the problem of learning a sequence of tasks using only the data from the most recent environment, aiming at obtaining a reasonable performance on the whole sequence. Existing works consider solutions to preserve the knowledge of the previous tasks either by keeping the corresponding system predictions unchanged during training of the new task, or by keeping the model parameters in a neighborhood of the sequence of the previous optimal weights. While the first suffers from the difference in the task distributions, the second needs to store a large number of parameters. 

The solution presented here reduces forgetting of earlier tasks by controlling the distance between the representations of the different tasks. Rather than preserving the optimal weights of the previous tasks, we propose an alternative that preserves the features that are crucial for the performance in the corresponding environments. Undercomplete autoencoders are used to learn the submanifold that represents these important features. The method is tested on image classification problems, in sequences of two or three tasks, starting either from a small or a large dataset. An improvement in performance over the state-of-the-art is achieved in all the tested scenarios. Especially, we showed a better preservation of the old tasks.

Despite the demonstrated improvements, this work also identifies possible further developments. A direction that is worth exploring is to use the autoencoders as data generators rather than relying on the new data. This would give a stronger solution in the situation where the new data does not represent previous distributions well.
\newline 
{\textbf{Acknowledgment:}
 The second equal author's PhD is funded by an FWO scholarship. This work is partially funded by Internal Funds KU Leuven, FP7-MC-CIG
334380, and an Amazon Academic Research Award.}




{\small
\bibliographystyle{ieee}
\bibliography{egbib}

\begin{thebibliography}{10}\itemsep=-1pt

\bibitem{aljundi2016expert}
R.~Aljundi, P.~Chakravarty, and T.~Tuytelaars.
\newblock Expert gate: Lifelong learning with a network of experts.
\newblock {\em arXiv preprint arXiv:1611.06194}, 2016.

\bibitem{bottou2010large}
L.~Bottou.
\newblock Large-scale machine learning with stochastic gradient descent.
\newblock In {\em Proceedings of COMPSTAT'2010}, pages 177--186. Springer,
  2010.

\bibitem{bourlard1988auto}
H.~Bourlard and Y.~Kamp.
\newblock Auto-association by multilayer perceptrons and singular value
  decomposition.
\newblock {\em Biological cybernetics}, 59(4-5):291--294, 1988.

\bibitem{caruana1998multitask}
R.~Caruana.
\newblock Multitask learning.
\newblock In {\em Learning to learn}, pages 95--133. Springer, 1998.

\bibitem{donahue2014decaf}
J.~Donahue, Y.~Jia, O.~Vinyals, J.~Hoffman, N.~Zhang, E.~Tzeng, and T.~Darrell.
\newblock Decaf: A deep convolutional activation feature for generic visual
  recognition.
\newblock In {\em Icml}, volume~32, pages 647--655, 2014.

\bibitem{french1999catastrophic}
R.~M. French.
\newblock Catastrophic forgetting in connectionist networks.
\newblock {\em Trends in cognitive sciences}, 3(4):128--135, 1999.

\bibitem{girshick2014rich}
R.~Girshick, J.~Donahue, T.~Darrell, and J.~Malik.
\newblock Rich feature hierarchies for accurate object detection and semantic
  segmentation.
\newblock In {\em Proceedings of the IEEE conference on computer vision and
  pattern recognition}, pages 580--587, 2014.

\bibitem{Goodfellow-et-al-2016-Book}
I.~Goodfellow, Y.~Bengio, and A.~Courville.
\newblock Deep learning.
\newblock Book in preparation for MIT Press, 2016.

\bibitem{goodfellow2013empirical}
I.~J. Goodfellow, M.~Mirza, D.~Xiao, A.~Courville, and Y.~Bengio.
\newblock An empirical investigation of catastrophic forgetting in
  gradient-based neural networks.
\newblock {\em arXiv preprint arXiv:1312.6211}, 2013.

\bibitem{hinton2015distilling}
G.~Hinton, O.~Vinyals, and J.~Dean.
\newblock Distilling the knowledge in a neural network.
\newblock {\em arXiv preprint arXiv:1503.02531}, 2015.

\bibitem{kirkpatrick2016overcoming}
J.~Kirkpatrick, R.~Pascanu, N.~Rabinowitz, J.~Veness, G.~Desjardins, A.~A.
  Rusu, K.~Milan, J.~Quan, T.~Ramalho, A.~Grabska-Barwinska, et~al.
\newblock Overcoming catastrophic forgetting in neural networks.
\newblock {\em arXiv preprint arXiv:1612.00796}, 2016.

\bibitem{krizhevsky2012imagenet}
A.~Krizhevsky, I.~Sutskever, and G.~E. Hinton.
\newblock Imagenet classification with deep convolutional neural networks.
\newblock In {\em Advances in neural information processing systems}, pages
  1097--1105, 2012.

\bibitem{li2016learning}
Z.~Li and D.~Hoiem.
\newblock Learning without forgetting.
\newblock In {\em European Conference on Computer Vision}, pages 614--629.
  Springer, 2016.

\bibitem{mcclelland1995there}
J.~L. McClelland, B.~L. McNaughton, and R.~C. O'reilly.
\newblock Why there are complementary learning systems in the hippocampus and
  neocortex: insights from the successes and failures of connectionist models
  of learning and memory.
\newblock {\em Psychological review}, 102(3):419, 1995.

\bibitem{mccloskey1989catastrophic}
M.~McCloskey and N.~J. Cohen.
\newblock Catastrophic interference in connectionist networks: The sequential
  learning problem.
\newblock {\em Psychology of learning and motivation}, 24:109--165, 1989.

\bibitem{mitchell1980need}
T.~M. Mitchell.
\newblock {\em The need for biases in learning generalizations}.
\newblock Department of Computer Science, Laboratory for Computer Science
  Research, Rutgers Univ. New Jersey, 1980.

\bibitem{Nilsback08}
M.-E. Nilsback and A.~Zisserman.
\newblock Automated flower classification over a large number of classes.
\newblock In {\em Proceedings of the Indian Conference on Computer Vision,
  Graphics and Image Processing}, Dec 2008.

\bibitem{pentina2015lifelong}
A.~Pentina and C.~H. Lampert.
\newblock Lifelong learning with non-iid tasks.
\newblock In {\em Advances in Neural Information Processing Systems}, pages
  1540--1548, 2015.

\bibitem{quattoni2009recognizing}
A.~Quattoni and A.~Torralba.
\newblock Recognizing indoor scenes.
\newblock In {\em Computer Vision and Pattern Recognition, 2009. CVPR 2009.
  IEEE Conference on}, pages 413--420. IEEE, 2009.

\bibitem{ratcliff1990connectionist}
R.~Ratcliff.
\newblock Connectionist models of recognition memory: Constraints imposed by
  learning and forgetting functions.
\newblock {\em Psychological review}, 97(2):285--308, 1990.

\bibitem{rebuffi2016icarl}
S.-A. Rebuffi, A.~Kolesnikov, and C.~H. Lampert.
\newblock icarl: Incremental classifier and representation learning.
\newblock {\em arXiv preprint arXiv:1611.07725}, 2016.

\bibitem{russakovsky2015imagenet}
O.~Russakovsky, J.~Deng, H.~Su, J.~Krause, S.~Satheesh, S.~Ma, Z.~Huang,
  A.~Karpathy, A.~Khosla, M.~Bernstein, et~al.
\newblock Imagenet large scale visual recognition challenge.
\newblock {\em International Journal of Computer Vision}, 115(3):211--252,
  2015.

\bibitem{rusu2016progressive}
A.~A. Rusu, N.~C. Rabinowitz, G.~Desjardins, H.~Soyer, J.~Kirkpatrick,
  K.~Kavukcuoglu, R.~Pascanu, and R.~Hadsell.
\newblock Progressive neural networks.
\newblock {\em arXiv preprint arXiv:1606.04671}, 2016.

\bibitem{silver2002task}
D.~L. Silver and R.~E. Mercer.
\newblock The task rehearsal method of life-long learning: Overcoming
  impoverished data.
\newblock In {\em Conference of the Canadian Society for Computational Studies
  of Intelligence}, pages 90--101. Springer, 2002.

\bibitem{silver2013lifelong}
D.~L. Silver, Q.~Yang, and L.~Li.
\newblock Lifelong machine learning systems: Beyond learning algorithms.
\newblock In {\em AAAI Spring Symposium: Lifelong Machine Learning}, pages
  49--55. Citeseer, 2013.

\bibitem{simonyan2014very}
K.~Simonyan and A.~Zisserman.
\newblock Very deep convolutional networks for large-scale image recognition.
\newblock {\em arXiv preprint arXiv:1409.1556}, 2014.

\bibitem{WelinderEtal2010}
P.~Welinder, S.~Branson, T.~Mita, C.~Wah, F.~Schroff, S.~Belongie, and
  P.~Perona.
\newblock {Caltech-UCSD Birds 200}.
\newblock Technical Report CNS-TR-2010-001, California Institute of Technology,
  2010.

\bibitem{zeiler2012adadelta}
M.~D. Zeiler.
\newblock Adadelta: an adaptive learning rate method.
\newblock {\em arXiv preprint arXiv:1212.5701}, 2012.

\end{thebibliography}
}
\newpage
\appendix
\section{Analysis of the main method}
In this section, we present a detailed analysis of the main contribution of the paper and demonstrate its theoretical grounding. 

The following derivations are based on the hypothesis that all the functions involved in the model training are Lipschitz continuous. 
The most commonly used functions such as sigmoid, ReLU or linear functions satisfy this condition. It is also the case for the most commonly used loss functions, e.g.\ softmax, logistic, or hinge losses.  Note that squared and exponential losses are not Lipschitz continuous, but this is not a significant limitation as such losses are less frequently applied in practice due to their sensitivity to label noise.
\begin{definition}
We say that a function $f$ is Lipschitz continuous if and only if there exist a constant $K$ such that:
\begin{equation*}
\forall (x,y), \|f(x) - f(y)\| \le K \|x-y\|
\end{equation*}
\end{definition}

\subsection{Relation between Encoder based Lifelong Learning and joint training}
In the sequel, we use the same notation as in the main paper.
In a two-task scenario, we propose to use the following objective to train a network using only the data from the second task:
\begin{align}
\mathcal{R} &= \mathbb{E}[\ell(T_2\circ T\circ  F(\mathcal{X}^{(2)}), \mathcal{Y}^{(2)}) )\nonumber \\
&+  \ell_{dist}(T_1\circ T\circ F(\mathcal{X}^{(2)}), T_1^*\circ T^*\circ F^*(\mathcal{X}^{(2)})  ) \nonumber  \\
& + \frac{\alpha}{2} \| \sigma(W_{enc}F(\mathcal{X}^{(2)})) - \sigma(W_{enc}F^*(\mathcal{X}^{(2)})) \|_2^2] \label{eq:CodeLoss}.
\end{align}
In the ideal case where we can keep in memory data from the first task, the best solution is to train the network jointly by minimizing the following objective:
\begin{equation}
\mathbb{E}[\ell(T_2\circ T\circ  F(\mathcal{X}^{(2)}), \mathcal{Y}^{(2)}) )] + \mathbb{E}[\ell(T_1\circ T\circ  F(\mathcal{X}^{(1)}), \mathcal{Y}^{(1)}) ) ]
\label{eq:StatRisk}
\end{equation}

\begin{proposition}
  The difference between \eqref{eq:StatRisk} and $\mathbb{E}[\ell(T_2\circ~T\circ~F(\mathcal{X}^{(2)}),\mathcal{Y}^{(2)}))
    +  \ell(T_1\circ T\circ F(\mathcal{X}^{(2)}), T_1^*\circ T^*\circ F^*(\mathcal{X}^{(2)})  )]$ can be controlled by the independent minimization of the knowledge distillation loss and five terms: \eqref{eq:X1Feat}, \eqref{eq:AEloss}, \eqref{eq:rec1-rec2}, \eqref{eq:rec2*-rec2}, and \eqref{eq:recErrX2}.
\end{proposition}

\begin{proof}
From Lipschitz continuity, we deduce that the difference between \eqref{eq:StatRisk} and $\mathbb{E}[\ell(T_2\circ~T\circ~F(\mathcal{X}^{(2)}),\mathcal{Y}^{(2)})) +  \ell(T_1\circ T\circ F(\mathcal{X}^{(2)}), T_1^*\circ T^*\circ F^*(\mathcal{X}^{(2)})  )]$ is bounded, and we can write for any vector norm $\|.\|$: 
\begin{align}
| \ell(T_1\circ T\circ & F(\mathcal{X}^{(1)}), \mathcal{Y}^{(1)}) \nonumber\\
&-  \ell(T_1\circ T\circ F(\mathcal{X}^{(2)}), T_1^*\circ T^*\circ F^*(\mathcal{X}^{(2)})  )| \nonumber\\
&\le K_1 \|F(\mathcal{X}^{(1)})-F(\mathcal{X}^{(2)})\|\label{eq:FeatDist}\\
& + K_2\| \mathcal{Y}^{(1)}- T^*_1\circ T^* \circ F^*(\mathcal{X}^{(2)})\|\label{eq:TargetDist}.
\end{align}
\eqref{eq:TargetDist} is related to the classification error on the first task. Indeed, using the triangle inequality, we can write:
\begin{align}
\| \mathcal{Y}^{(1)}&- T^*_1\circ T^* \circ F^*(\mathcal{X}^{(2)})\| \nonumber\\
&\le \| \mathcal{Y}^{(1)}- T^*_1\circ T^* \circ F^*(\mathcal{X}^{(1)})\| \label{eq:LossT1} \\
&+ \| T^*_1\circ T^* \circ F^*(\mathcal{X}^{(1)})- T^*_1\circ T^* \circ F^*(\mathcal{X}^{(2)})\| \label{eq:DistTarget}
\end{align}
Note that all the terms on the right hand side of the inequality do not change during training of task 2, and thus cannot be controlled during the second training phase. Moreover, \eqref{eq:LossT1} depends only on the capacity of the network and is therefore not influenced by the encoder based lifelong learning scheme. \eqref{eq:DistTarget} is the result of using $\mathcal{X}^{(2)}$ instead of $\mathcal{X}^{(1)}$. In order to reduce the effect of this shift, we use the knowledge distillation loss~\cite{hinton2015distilling} as in the Learning without Forgetting (LwF) method~\cite{li2016learning}.

On the other hand, expression~\eqref{eq:FeatDist} is bounded as well (using again the triangle inequality):
\begin{align}
\|F(\mathcal{X}^{(1)})-F(\mathcal{X}^{(2)})\| &\le
\| F(\mathcal{X}^{(1)}) - F^*(\mathcal{X}^{(1)}) \| \label{eq:X1Feat} \\
&+ \| F^*(\mathcal{X}^{(1)}) - r\circ F^*(\mathcal{X}^{(1)}) \| \label{eq:AEloss}\\
&+  \| r\circ F^*(\mathcal{X}^{(1)}) - r\circ F^*(\mathcal{X}^{(2)}) \|\label{eq:rec1-rec2} \\
&+  \| r\circ F^*(\mathcal{X}^{(2)}) - r\circ F(\mathcal{X}^{(2)}) \|\label{eq:rec2*-rec2}\\
&+  \| r\circ F(\mathcal{X}^{(2)}) - F(\mathcal{X}^{(2)})\|\label{eq:recErrX2}.
\end{align}

This bound generalizes trivially to the expected value of the loss, which finishes the proof. 
\end{proof}

Analyzing each of these terms individually, we see that our training strategy effectively controls the difference with the joint training risk: 
\begin{itemize}
\item \eqref{eq:AEloss} is minimized through the autoencoder (AE) training and does not change during training of task 2.
\item \eqref{eq:rec1-rec2} does not change during training of task 2. Moreover, $\|~r\circ~F^*(\mathcal{X}^{(1)})~-~r\circ~F^*(\mathcal{X}^{(2)})\|$ measures the distance between two elements of the manifold that the AE represents. As the use of weight decay during training the task model makes the weights small, the projection of data into the feature space is contractive.  Through our experiments, we observed that $r$ is also contractive. As a result, this distance is significantly smaller than $\|\mathcal{X}^{(1)} - \mathcal{X}^{(2)}\|$. The experiment in Sec.~\ref{Sec:ExpEncoderDist} supports this observation.
\item \textbf{\eqref{eq:rec2*-rec2} is controlled during the training thanks to the second constraint we propose to use}.
\item \eqref{eq:recErrX2} is the part of the features that we propose to relax in order to give space for the model to adjust to the second task as explained in Sec.~3.3 in the main text. Indeed, if we control this distance, the features are forced to stay in the manifold related to the first task. This will result in a stronger conservation of the first task performance, but the training model will face the risk of not being able to converge for the second task.
\item \eqref{eq:X1Feat} is a term that we cannot access during training. However, we observed that this distance is either decreasing or first increasing then decreasing during the training. An explanation of this behavior is that in the beginning $\ell(T_2\circ~T\circ~F(\mathcal{X}^{(2)}),\mathcal{Y}^{(2)}))$ may have a bigger influence on the objective, however, after decreasing the loss for the second task, $ \ell(T_1\circ~T\circ~F(\mathcal{X}^{(2)}),T_1^*\circ~T^*\circ~F^*(\mathcal{X}^{(2)}))$ and $\|\sigma(W_{enc}F(\mathcal{X}^{(2)}))~-~\sigma(W_{enc}F^*(\mathcal{X}^{(2)}))\|_2$ tend to push $F$ towards $F^*$. Figure 5 in the main text, and the paragraph ``Behavior analysis" support this observation.  
\end{itemize}
This derivation motivates the code distance that we propose to use. It also elucidates the sources of possible divergence from joint-training. 
   
   
   

\subsection{Empirical study: $F$ and $r$ are contractive}
\label{Sec:ExpEncoderDist}
This experiment aims to show empirically that $\| \mathcal{X}^{(1)} - \mathcal{X}^{(2)} \|$ is significantly larger than $\| r\circ F^*(\mathcal{X}^{(1)}) - r\circ F^*(\mathcal{X}^{(2)}) \|$ using the $\ell_2$ norm. To have empirical evidence for this observation, we conducted the following experiment:
\begin{enumerate}
\item First, we generate random inputs for the model from two different distributions. We use normal distributions with uniformly distributed means and variances. 
\item Then, we compute the features of these inputs (output of $F$).
\item These features are fed to the AE to compute the reconstructions. 
\item The mean squared error (MSE) between the samples, the features and the reconstructions are stored. 
\item This procedure is repeated for several trials. Each time a different pair of distributions is used. Finally, the mean of the obtained MSE is computed.
\end{enumerate}

We repeat this experiment twice. In the first instance, the mean and variance of the Gaussian distributions are generated in a way to have relatively small distances between the samples, and in the second we force the samples to have bigger distance. Tables~\ref{Tab:F_r_contractive_flowers} and ~\ref{Tab:F_r_contractive_imagenet} show the results of this experiment. The numbers in Table~\ref{Tab:F_r_contractive_flowers} are computed using AlexNet fine-tuned on Flowers dataset, and the AE trained as explained in Sec.~3.3 (in the main text) on the features of Flowers extracted using AlexNet convolutional layers after convergence. We report in this table the mean MSE and the error bars obtained with 50 generated samples from 500 different pairs of Gaussian distributions. The numbers in Table~\ref{Tab:F_r_contractive_imagenet} are computed similarly but from the AlexNet and the AE related to ImageNet.  In all cases, the difference between the samples is many orders of magnitude larger than the difference between the reconstructions indicating that the mapping is indeed contractive, and the residual error from this step is minimal.

\begin{table}
\begin{tabular}{llll}
\toprule
 & Samples & Features  & Reconstructions \\
\midrule
Exp. 1 & 229.98 & 33.05 & 0.21 \\
& $\pm 0.065$ & $\pm 0.066$ & $\pm 0.0024$ \\
Exp. 2 & 10176.54  & 106.60  & 0.21 \\
& $\pm 41.17$ & $\pm 0.32$ & $\pm 0.0027$ \\
\bottomrule
\end{tabular}
\caption{Showing that $F$ and $r$ are contractive: Mean MSE of 50 samples  from random Gaussian distribution, of their corresponding features and reconstructions over 500 trails - The main model and the AE are trained on Flowers dataset. As we move from left to right in the table, the entries in the columns decrease significantly verifying that the mappings are contractive.}
\label{Tab:F_r_contractive_flowers}
\end{table}

\begin{table}
\begin{tabular}{llll}
\toprule
 & Samples & Features  & Reconstructions \\
\midrule
Exp. 1 & 230.02 & 9.48 & 1.55 \\
& $\pm 0.063 $ & $\pm 0.022$ & $\pm 0.01$ \\
Exp. 2 & 10173.11 & 61.29  & 2.12\\
& $\pm 41.35 $ & $\pm 0.18$ & $\pm 0.013$ \\
\bottomrule
\end{tabular}
\caption{Showing that $F$ and $r$ are contractive: Mean MSE of 50 samples from random Gaussian distribution, of their corresponding features and reconstructions over 500 trails - The main model and the AE are trained on Imagenet dataset. As we move from left to right in the table, the entries in the columns decrease significantly verifying that the mappings are contractive.}
\label{Tab:F_r_contractive_imagenet}
\end{table}

\subsection{Multiple task scenario}
 Each time a task is added, a new source of divergence from the joint-training objective is added. The difference with \eqref{eq:StatRisk}  grows with $\mathcal{T}$. Indeed, at each step, an extra irreducible $\|r\circ F^*(\mathcal{X}^{(\mathcal{T}-1)}) -  r\circ F^*(\mathcal{X}^{(\mathcal{T})})\|$ is added. Moreover, for each task, the autoencoders of the previous tasks are trained using the corresponding feature extractors. The remaining difference between the two losses that is neither directly controlled nor decreasing while training the model can be expressed as follows (for a constant $K$): 
\begin{align}
K\Biggl(&\sum_{t=2}^\mathcal{T} \|r_{t-1}\circ F^{(*,t-1)}(\mathcal{X}^{(t-1)}) - r_{t-1}\circ F^{(*,t-1)}(\mathcal{X}^{(t)})\| \label{eq:RecsDistances} \\
& +  \sum_{t=2}^{\mathcal{T}} \|r_{t-1}\circ F(\mathcal{X}^{(\mathcal{T})}) - F(\mathcal{X}^{(\mathcal{T})})\| \label{eq:UncontrolledDistances}\\
& + \sum_{t=1}^{\mathcal{T}-2} \|r_{t}\circ F^{(*,\mathcal{T}-1)}(\mathcal{X}^{(t)}) - r_t\circ F^{(*,t)}(\mathcal{X}^{(t)})\|\Biggr)\label{eq:FeatsDistances},
\end{align}

where $F^{(*,t)}$ is the feature extraction operator after training the model on the task $t$ data. 
As observed for~\eqref{eq:rec1-rec2}, expressions \eqref{eq:RecsDistances} and \eqref{eq:FeatsDistances} remain small and the conclusion of the experiment in Sec.~\ref{Sec:ExpEncoderDist} holds also for these distances. Therefore, their effect on the growth of the difference with joint-training is minimal.  
In contrast to the other terms, as for~\eqref{eq:recErrX2}, controlling \eqref{eq:UncontrolledDistances} may prevent the network from converging for the new task. Thus, relaxing these distances is an important degree of freedom for our method.

\section{Additional experiments}

The experiments in this section aim to show that the success of the proposed method is independent of the chosen model. For this purpose, we train VGG-verydeep-16~\cite{simonyan2014very} in a two-task image classification scenario and compare our method against the state-of-the-art (LwF)~\cite{li2016learning}.

\paragraph{Tested scenario} We test our method and LwF on a two-task scenario starting from ImageNet (LSVRC 2012 subset)~\cite{russakovsky2015imagenet} (more than 1 million training images) then training on  MIT \textit{Scenes} \cite{quattoni2009recognizing} for indoor scene classification (5,360 samples). The showed results are obtained on the test data of Scenes and the validation data of ImageNet.

\paragraph{Architecture}
\begin{figure*}[!ht]
\centering
\includegraphics[width=0.7\textwidth]{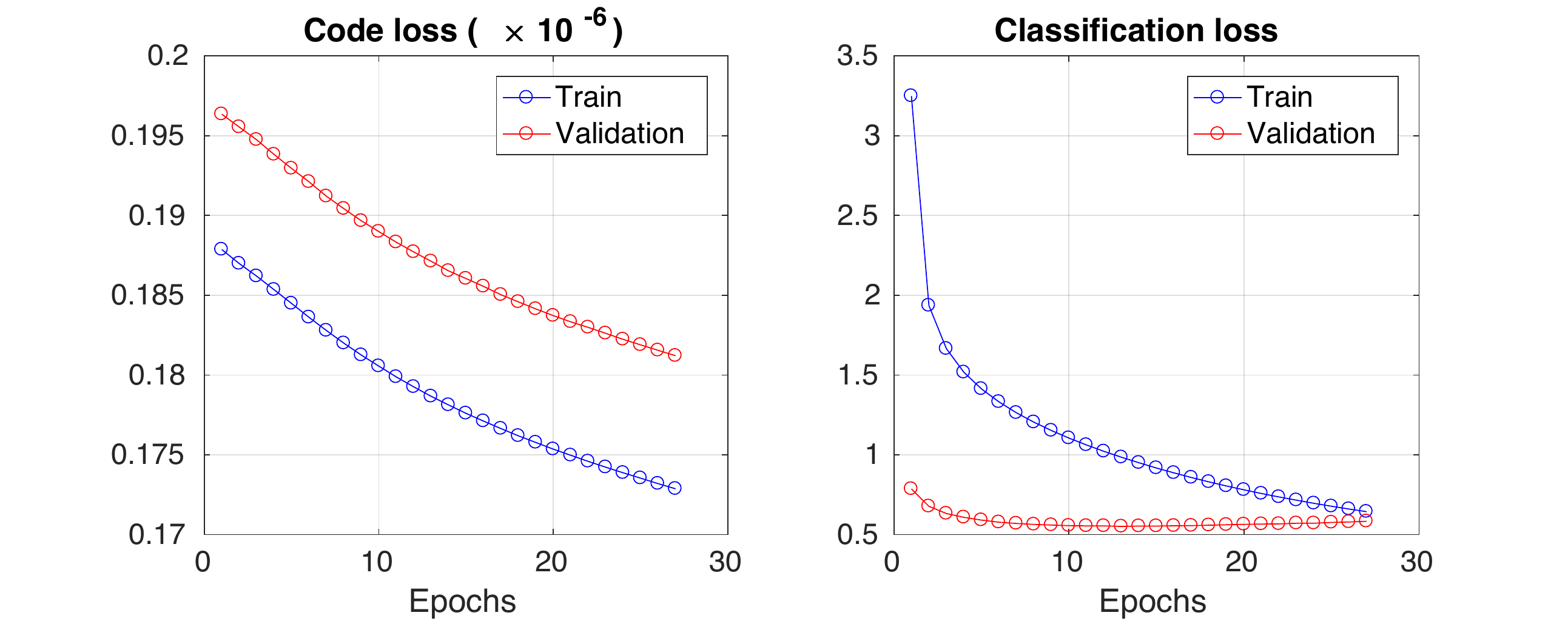}
\caption{Training of VGG-verydeep-16 based autoencoder for ImageNet - The objective makes the code loss \emph{and} the classification loss decrease. The training is stopped when we observe a convergence of the classification loss.  }
\label{fig:VGG_AE_loss}
\end{figure*}
We experiment with VGG-verydeep-16~\cite{simonyan2014very} due to its popularity and success in many image classification tasks. The feature extraction block $F$ corresponds to the convolutional layers. Note that this architecture has feature extractor twice as deep as AlexNet~\cite{krizhevsky2012imagenet} (used in the main text). As for the experiments conducted using AlexNet, the shared task operator $T$ corresponds to all but the last fully connected layers (i.e., $fc6$ and $fc7$), while the task-specific part $T_i$ contains the last classification layer ($fc8$). The used hyperparameters are the same as in the main text: same $\alpha$  ($10^{-3}$) for the training of our method, and same  $\lambda$ ($10^{-6}$) for the autoencoder training. The used architecture for the autoencoder is also similar to the main text (2-layers with a sigmoid non-linearity in between) with a code of 300 entries. Figure~\ref{fig:VGG_AE_loss} shows the evolution of the code error and the classification error during the training of the autoencoder on ImageNet features extracted using VGG-verydeep-16.

\paragraph{Autoencoder size}
To illustrate the growth of the size of the autoencoder with the size of the model, we show in Table~\ref{tab:AE_size} the memory required by AlexNet and VGG-verydeep-16 while training for Scenes after ImageNet, along with the autoencoder input length (Feature size) and the memory required by the autoencoder during training with our method. Naturally, the autoencoder size grows with the feature length, but remains very small comparing with the size of the global model. 

\begin{table}
\begin{tabular}{llll}
\toprule
Model & Size & Feature size  & Autoencoder size \\
\midrule
AlexNet & 449 MB & 9216 & 10 MB (2.2\%) \\
VGG & 1.1 GB & 25088  & 28 MB (2.5\%)\\
\bottomrule
\end{tabular}
\caption{Size of the used autoencoders compared to the size of the models during training. The model size corresponds to the required memory during training. The feature size corresponds to the length of the feature extractor output.}
\label{tab:AE_size}
\end{table}

\paragraph{Results: Method behavior comparison}
\begin{figure*}[!ht]
\centering
\includegraphics[width= \textwidth]{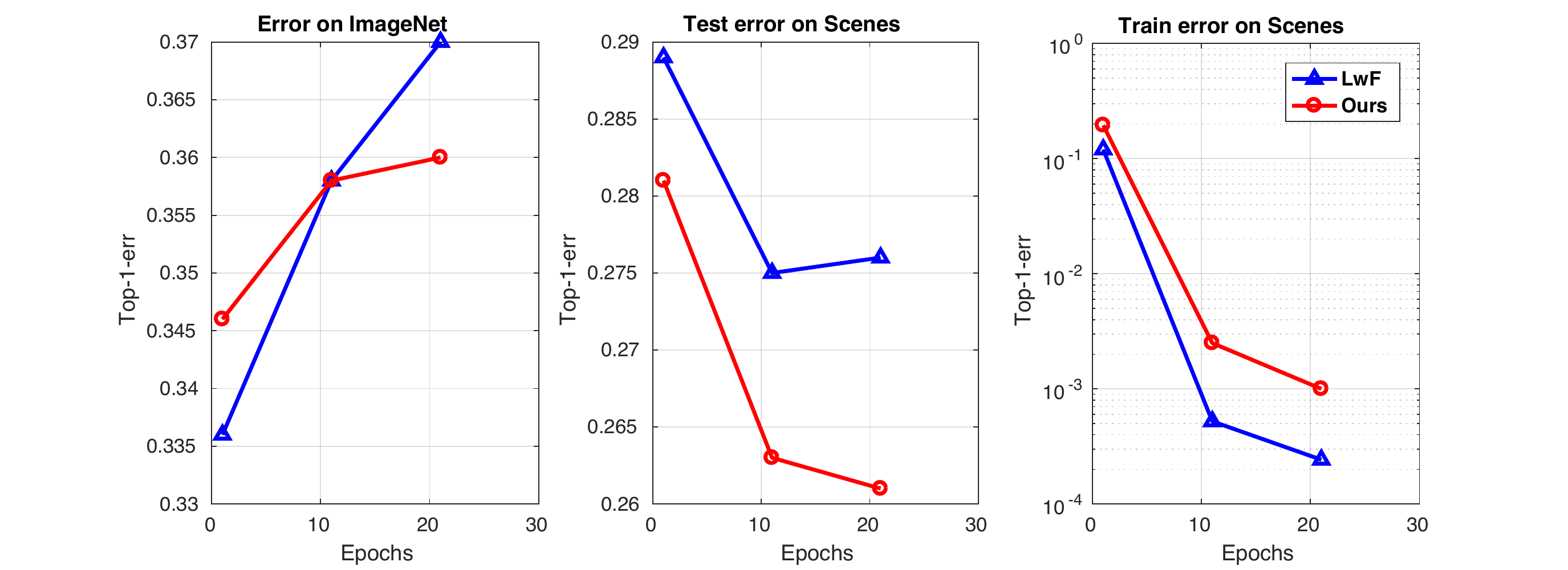}
\caption{Comparison between our method (red) and LwF (blue). Left : Evolution of the error on ImageNet validation set; it shows a slower loss of performance on ImageNet for our method - Center: Evolution of the error on Scenes test set - Right: Evolution of the error on Scenes training set. Center and Right suggest that our method benefits from a better regularization than LwF.}
\label{fig:MethodCompare}
\end{figure*}
Figure~\ref{fig:MethodCompare} shows the evolution of the model performance for both tasks, ImageNet and Scenes, when trained with our method (red curves) and with LwF (blue curves).

Our method shows a better preservation of the performance on ImageNet. Even if the classification error grows for both methods, it increases slower in our case. After 20 epochs, the performance on ImageNet is 1\% higher using our method.

The test and train errors on Scenes highlight an interesting characteristic of our method. The use of the code loss on top of LwF appears to act as a regularizer for the training on Scenes. Our method shows a slightly higher training error, and a better generalization. VGG-verydeep-16 is a large model, and the risk of overfitting while training on a small dataset like Scenes is higher than for AlexNet. A stronger regularization (using a higher value of $\alpha$) may thus result in an improvement of the behavior of our method.

\paragraph{Conclusion} From this experiment, we observe that the convergence of the autoencoder training and the improvement observed over LwF are not dependent on the used architecture. Moreover, the additional memory required by our method remains small with respect to the size of the global model. 

\end{document}